\documentclass[a4paper,11pt]{article}

\usepackage{booktabs}
\usepackage{authblk}
\usepackage[margin=4cm,columnsep=.8cm]{geometry}

\usepackage{algorithm2e}
\usepackage{amsthm}
\usepackage{amsmath}
\usepackage{amssymb}
\usepackage{bm}
\usepackage{comment}
\usepackage{graphicx}
\usepackage{nameref}
\usepackage{slashbox}
\usepackage{soul}
\usepackage{xcolor}
\usepackage{subfigure}

\DeclareMathOperator*{\argmax}{arg\,max}

\newcommand{\event}{\mathcal{E}_T}
\newcommand{\eventc}{\overline{\mathcal{E}}_T}
\newcommand{\myvec}{\boldsymbol}
\newcommand{\hist}{\mathcal{H}_{t-1}}
\newcommand{\E}[1]{\mathbb{E}\left[#1\right]}

\newcommand{\locactsize}{\widetilde{A}}

\providecommand{\newdef}[2]{\newtheorem{#1}{#2}}
\newtheorem{theorem}{Theorem}
\newtheorem{assumption}{Assumption}
\newtheorem{lemma}{Lemma}
\newtheorem{corollary}{Corollary}
\theoremstyle{definition}
\newdef{definition}{Definition}

\begin{document}
\title{\textbf{Multi-Agent Thompson Sampling for Bandit Applications with Sparse Neighbourhood Structures}}

\author[1,2,*]{Timothy Verstraeten}
\author[1]{Eugenio Bargiacchi}
\author[1]{Pieter J.K. Libin}
\author[2]{Jan Helsen}
\author[1,3]{Diederik M. Roijers}
\author[1]{Ann Now\'e}
\affil[1]{Vrije Universiteit Brussel, Artificial Intelligence Lab Brussels, Elsene, 1050, Belgium}
\affil[2]{Vrije Universiteit Brussel, Acoustics and Vibrations Research Group, Elsene, 1050, Belgium}
\affil[3]{HU University of Applied Sciences, Institute for ICT, Utrecht, 3584CS, Netherlands}

\affil[*]{tiverstr@vub.be} 


\setcounter{Maxaffil}{0}
\renewcommand\Affilfont{\itshape\small}
\date{}
\maketitle

\begin{abstract}
Multi-agent coordination is prevalent in many real-world applications. However, such coordination is challenging due to its combinatorial nature. An important observation in this regard is that agents in the real world often only directly affect a limited set of neighbouring agents. Leveraging such loose couplings among agents is key to making coordination in multi-agent systems feasible.
In this work, we focus on \emph{learning} to coordinate.
Specifically, we consider the multi-agent multi-armed bandit framework, in which fully cooperative loosely-coupled agents must learn to coordinate their decisions to optimize a common objective. 
We propose multi-agent Thompson sampling (MATS), a new Bayesian exploration-exploitation algorithm that leverages loose couplings. We provide a regret bound that is sublinear in time and low-order polynomial in the highest number of actions of a single agent for sparse coordination graphs. Additionally, we empirically show that MATS outperforms the state-of-the-art algorithm, MAUCE, on two synthetic benchmarks, and a novel benchmark with Poisson distributions. An example of a loosely-coupled multi-agent system is a wind farm. Coordination within the wind farm is necessary to maximize power production. As upstream wind turbines only affect nearby downstream turbines, we can use MATS to efficiently learn the optimal control mechanism for the farm. To demonstrate the benefits of our method toward applications we apply MATS to a realistic wind farm control task. In this task, wind turbines must coordinate their alignments with respect to the incoming wind vector in order to optimize power production. Our results show that MATS improves significantly upon state-of-the-art coordination methods in terms of performance, demonstrating the value of using MATS in practical applications with sparse neighbourhood structures.
\end{abstract}

\section*{Introduction}

Multi-agent decision coordination is prevalent in many real-world applications, such as traffic light control \cite{wiering2000multi}, warehouse commissioning \cite{claes2017decentralised} and wind farm control \cite{gebraad2015maximum}. 
Often, such settings can be formulated as coordination problems in which agents have to cooperate in order to optimize a shared team reward \cite{Boutilier1996mmdp}.

Handling multi-agent settings is challenging, as the size of the joint action space scales exponentially with the number of agents in the system. Therefore, an approach that directly considers all agents' actions jointly is computationally intractable. This has made such coordination problems the central focus in the planning literature \cite{PIFMDP,VIFMDP,guestrin2002multiagent,guestrin2002context}.  Fortunately, in real-world settings agents often only directly affect a limited set of neighbouring agents. This means that the global reward received by all agents can be decomposed into local components that only depend on small subsets of agents. Exploiting such loose couplings is key in order to keep multi-agent decision problems tractable \cite{chapman2013convergent}.  

In this work, we consider learning to coordinate in multi-agent systems. For example, consider a wind farm control task, which is comprised of a set of wind turbines, and we aim to maximize the farm's total productivity. When upstream turbines directly face the incoming wind stream, energy is extracted from wind. This reduces the productivity of downstream turbines, potentially damaging the overall power production. However, turbines have the option to rotate, in order to deflect the turbulent flow away from turbines downwind \cite{vandijk2016}. Due to the complex nature of the aerodynamic interactions between the turbines, constructing a model of the environment and deriving a control policy using planning techniques is extremely challenging \cite{marden2013model}. Instead, a joint control policy among the turbines can be \emph{learned} to effectively maximize the productivity of the wind farm. The system is loosely coupled, as redirection only directly affects adjacent turbines.

While most of the literature only considers approximate reinforcement learning methods for learning in multi-agent systems, it has recently been shown \cite{bargiacchi2018learning} that it is possible to achieve theoretical bounds on the regret (i.e., how much reward is lost due to learning). In this work, we use the multi-agent multi-armed bandit problem definition, and improve upon the state of the art. 
Specifically, we propose the multi-agent Thompson sampling (MATS) algorithm, which exploits loosely-coupled interactions in multi-agent systems. The loose couplings are formalized as a \emph{coordination graph}, which defines for each pair of agents whether their actions depend on each other. We assume the graph structure is known beforehand, which is the case in many real-world applications with sparse agent interactions (e.g., wind farm control).

Our method leverages the exploration-exploitation mechanism of Thompson sampling (TS). TS has been shown to be highly competitive to other popular methods, e.g., UCB \cite{chapelle2011empirical}. Recently, theoretical guarantees on its regret have been established \cite{agrawal2012analysis}, which renders the method increasingly popular in the literature. Additionally, due to its Bayesian nature, problem-specific priors can be specified. We argue that this has strong relevance in many practical fields, such as advertisement selection \cite{chapelle2011empirical} and influenza mitigation \cite{libin2018bayesian, bfts2019}. 

We provide a finite-time Bayesian regret analysis and prove that the upper regret bound of MATS is low-order polynomial in the number of actions of a single agent for sparse coordination graphs (Corollary~\ref{cor:regret}). This is a significant improvement over the exponential bound of classic TS, which is obtained when the coordination graph is ignored \cite{agrawal2012analysis}. We show that MATS improves upon the state of the art in various synthetic settings. Finally, we demonstrate that MATS achieves high performance on a realistic wind farm control task, in which multiple wind turbines have to be jointly aligned to maximize the total power production.


\section*{Problem statement}
\label{sec:problem}

In this work, we adopt the multi-agent multi-armed bandit (MAMAB) setting \cite{bargiacchi2018learning,stranders2012dcops}. A MAMAB is similar to the multi-armed bandit formalism \cite{thompson1933likelihood}, but considers multiple agents factored into groups. When the agents have pulled a joint arm, each group receives a reward. The goal shared by all agents is to maximize the total sum of rewards. Formally,
\begin{definition}
A multi-agent multi-armed bandit (MAMAB) is a tuple $\langle \mathcal{D}, \mathcal{A}, f\rangle$ where
\begin{itemize}
\item $\mathcal{D}$ is the set of $m$ enumerated agents. This set is factorized into $\rho$, possibly overlapping, subsets of agents $\mathcal{D}^e$.
\item $\mathcal{A} = \mathcal{A}_1 \times \dots \times \mathcal{A} _{m}$ is the set of joint actions, or joint arms, which is the Cartesian product of the sets of actions $\mathcal{A}_i$ for each of the $m$ agents in $\mathcal{D}$. We denote $\mathcal{A}^e$ as the set of local joint actions, or local arms, for the group $\mathcal{D}^e$.
\item $f(\myvec{a})$ is a stochastic function providing a global reward when a joint arm, $\myvec{a} \in \mathcal{A}$, is pulled. The global reward function is decomposed into $\rho$ noisy, observable and independent local reward functions, i.e., $f(\myvec{a}) = \sum^\rho_{e=1} f^e(\myvec{a}^e)$. A local function $f^e$ only depends on the local arm $\myvec{a}^e$ of the subset of agents in $\mathcal{D}^e$.
\end{itemize}
We denote the mean reward of a joint arm as $\mu(\myvec{a}) = \sum^\rho_{e=1} \mu^e(\myvec{a}^e)$. For simplicity, we refer to the $i^\text{th}$ agent by its index $i$.
\label{def:mamab}
\end{definition}
The dependencies between the local reward functions and the agents are described as a coordination graph \cite{guestrin2002multiagent}.
\begin{definition}
A coordination graph is a bipartite graph $G = \langle \mathcal{D}, \{f^e\}^\rho_{e=1}, E\rangle$, whose nodes $\mathcal{D}$ are agents and components of a factored reward function $f = \sum^\rho_{e=1} f^e$, and an edge $(i, f^e) \in E$ exists if and only if agent $i$ influences component $f^e$.
\end{definition}
\noindent
The dependencies in a MAMAB can be described by setting $E = \{(i, f^e)\ |\ i \in \mathcal{D}^e\}$.

In this setting, the objective is to minimize the expected cumulative regret \cite{agrawal2013further}, which is the cost incurred when pulling a particular joint arm instead of the optimal one.
\begin{definition}
The expected cumulative regret of pulling a sequence of joint arms until time step $T$ according to policy $\pi$ is
\begin{equation}
\begin{split}
\mathbb{E}\left[R(T, \pi)\right] &\triangleq \mathbb{E}\left[\sum^T_{t=1} \Delta(\myvec{a}_t)\ \middle|\ \pi\right]
\end{split}
\end{equation}
with
\begin{equation}
\begin{split}
\Delta(\myvec{a}_t) &\triangleq \mu(\myvec{a}_*) - \mu(\myvec{a}_t)\\
&= \sum^\rho_{e=1} \mu^e(\myvec{a}^e_*) - \mu^e(\myvec{a}^e_t),
\end{split}
\end{equation}
where $\myvec{a}_*$ is the optimal joint arm and $\myvec{a}_t$ is the joint arm pulled at time $t$. For the sake of brevity, we will omit $\pi$ when the context is clear.
\end{definition}

Cumulative regret can be minimized by using a policy that considers the full joint arm space, thereby ignoring loose couplings between agents. This leads to a combinatorial problem, as the joint arm space scales exponentially with the number of agents. Therefore, loose couplings need to be taken into account whenever possible.

\section*{Multi-agent Thompson sampling}
\label{sec:algo}

We propose the multi-agent Thompson sampling (MATS) algorithm for decision making in loosely-coupled multi-agent multi-armed bandit problems.
Consider a MAMAB with groups $\mathcal{D}^e$ (Definition~\ref{def:mamab}). The local means $\mu^e(\myvec{a}^e)$ are treated as unknown. According to the Bayesian formalism, we exert our beliefs over the local means $\mu^e(\myvec{a}^e)$ in the form of a prior, $Q^e_{\myvec{a}^e}(\cdot)$. At each time step $t$, MATS draws a sample $\mu^e_t(\myvec{a}^e)$ from the posterior for each group and local arm given the history, $\mathcal{H}_{t-1}$, consisting of local actions and rewards associated with past pulls:
\begin{equation}
\begin{split}
\mu^e_t(\myvec{a}^e) &\sim Q^e_{\myvec{a}^e}(\cdot\ |\ \mathcal{H}^e_{t-1})\\
 \mathcal{H}^e_{t-1} &\triangleq \{(\myvec{a}^e_i, f^e_i(\myvec{a}^e_i))\}^{t-1}_{i=1}.\\
\end{split}
\end{equation}
Note that during this step, MATS samples directly the posterior over the unknown local means, which implies that the sample $\mu^e_t(\myvec{a}^e)$ and the unknown mean $\mu^e(\myvec{a}^e)$ are independent and identically distributed at time step $t$.

Thompson sampling (TS) chooses the arm with the highest sample, i.e.,
\begin{equation}
\myvec{a}_t = \argmax_{\myvec{a}} \mu_t(\myvec{a}).
\end{equation}
However, in our case, the expected reward is decomposed into several local means. As conflicts between overlapping groups will arise, the optimal local arms for an agent in two groups may differ.
Therefore, we must define the argmax-operator to deal with the factored representation of a MAMAB, while still returning the full joint arm that maximizes the sum of samples, i.e.,
\begin{equation}
\myvec{a}_t = \argmax_{\myvec{a}} \sum^\rho_{e=1} \mu^e_t(\myvec{a}^e).
\label{eq:max_arm}
\end{equation}
To this end, we use variable elimination (VE), which computes the joint arm that maximizes the global reward without explicitly enumerating over the full joint arm space \cite{guestrin2002multiagent}. Specifically, VE consecutively eliminates an agent from the coordination graph, while computing its best response with respect to its neighbours. 
VE is guaranteed to return the optimal joint arm and has a computational complexity that is combinatorial in terms of the induced width of the graph, i.e., the number of neighbours of an agent at the time of its elimination. However, as the method is typically applied to a loosely-coupled coordination graph, the induced width is generally much smaller than the size of the full joint action space, which renders the maximization problem tractable \cite{guestrin2002multiagent,guestrin2002context}. Approximate efficient alternatives exist, such as max-plus \cite{vlassis2004anytime}, but using them will invalidate the proof for the Bayesian regret bound (Theorem~\ref{theorem:regret}).

Finally, the joint arm that maximizes Equation~\ref{eq:max_arm}, $\myvec{a}_t$, is pulled and a reward $f^e_t(\myvec{a}^e_t)$ will be obtained for each group. MATS is formally described in Algorithm~\ref{algo:mats}.
\begin{algorithm}
\KwData{Prior $Q^e_{\myvec{a}^e}$ per group $\mathcal{D}^e$ and local action $\myvec{a}^e$}
$\mathcal{H}_0 \leftarrow \{\}$\\
\For{$t \in [1..T]$}{
	$\forall e \in \left[1..\rho\right], \myvec{a}^e \in \mathcal{A}^e:$\\
	\qquad $\mu^e_t(\myvec{a}^e) \sim Q^e_{\myvec{a}^e}(\ \cdot\ |\ \mathcal{H}_{t-1})$\\
	$\myvec{a}_t \leftarrow \argmax_{\myvec{a}} \sum^\rho_{e=1} \mu^e_t(\myvec{a}^e)$ using VE\\
	$\langle f^e_t(\myvec{a}^e_t)\rangle^\rho_{e=1} \leftarrow $ Pull joint arm $\myvec{a}_t$\\
	$\mathcal{H}_t \leftarrow \mathcal{H}_{t-1} \cup \left\{\langle\myvec{a}^e_t, f^e_t(\myvec{a}^e_t)\rangle^\rho_{e=1}\right\}$\\
}
\caption{MATS}
\label{algo:mats}
\end{algorithm}

MATS belongs to the class of probability matching methods \cite{lattimore2018bandit}. 
\begin{definition} Given history $\mathcal{H}_{t-1}$, the probability distribution of the pulled arm $\myvec{a}_t$ is equal to the probability distribution of the optimal arm $\myvec{a}_*$. Formally,
\begin{equation}
P(\myvec{a}_t = \cdot\ |\ \mathcal{H}_{t-1}) = P(\myvec{a}_* = \cdot\ |\ \mathcal{H}_{t-1}).
\label{eq:prob_match}
\end{equation}
\label{def:prob_match}
\end{definition}
\noindent
Intuitively, MATS samples the local mean rewards according to the beliefs of the user at each time step, and maximizes over those means to find the optimal joint arm according to Definition~\ref{def:mamab}. This process is conceptually similar to traditional TS \cite{thompson1933likelihood}.

\section*{Bayesian regret analysis}
\label{sec:regret}

Many multi-agent systems are composed of locally connected agents. When formalized as a MAMAB (Definition~\ref{def:mamab}), our method is able to exploit these local structures during the decision process. We provide a regret bound for MATS that scales sublinearly with a factor $\locactsize T$, where $\locactsize$ is the number of local arms.

Consider a MAMAB $\langle\mathcal{D}, \mathcal{A}, f\rangle$ with $\rho$ groups and the following assumption on the rewards:
\begin{assumption}
The global rewards have a mean between 0 and 1, i.e.,$$\mu(\myvec{a}) \in [0, 1], \forall \myvec{a} \in \mathcal{A}.$$
\label{assumption:means}
\end{assumption}
\begin{assumption}
The local rewards shifted by their mean are $\sigma$-subgaussian distributed, i.e., $\forall e \in [1..\rho], \myvec{a}^e \in \mathcal{A}^e$,
$$\E{\exp\left(t(f^e(\myvec{a}^e) - \mu^e(\myvec{a}^e))\right)} \le \exp(0.5 \sigma^2 t^2).$$
\label{assumption:rewards}
\end{assumption}
We maintain the pull counters $n^e_{t-1}(a^e)$ and estimated means $\hat{\mu}^e_{t-1}(a^e)$ for local arms $a^e$.

Consider the event $\event$, which states that, until time step $T$, the differences between the local sample means and true means are bounded by a time-dependent threshold, i.e.,
\begin{equation}
\begin{split}
\event &\triangleq \left(\forall e, \myvec{a}^e, t\ :\ |\hat{\mu}^e_{t-1}(\myvec{a}^e) - \mu^e(\myvec{a}^e)| \le c^e_{t}(\myvec{a}^e)\right)\\
\end{split}
\label{eq:event}
\end{equation}
with
\begin{equation}
\begin{split}
c^e_t(\myvec{a}^e) &\triangleq \sqrt{\frac{2\sigma^2\log(\delta^{-1})}{n^e_{t-1}(\myvec{a}^e)}}.
\end{split}
\label{eq:threshold}
\end{equation}
where $\delta$ is a free parameter that will be chosen later. We denote the complement of the event by $\eventc$.
\begin{lemma}
\emph{(Concentration inequality)} The probability of exceeding the error bound on the local sample means is linearly bounded by $\tilde{A}T \delta$. Specifically,
\begin{equation}
P(\eventc) \le 2\tilde{A}T \delta.
\end{equation}
\label{lemma:conc_ineq}
\end{lemma}
\begin{proof}
Using the union bound (U), we can bound the probability of observing event $\eventc$ as
\begin{equation}
\begin{split}
P(\eventc) &\stackrel{(\ref{eq:event})}{=} P\left(\exists t, e, \myvec{a}^e\ :\ |\hat{\mu}^e_{t-1}(\myvec{a}^e) - \mu^e(\myvec{a}^e)| > c^e_t(\myvec{a}^e)\right)\\
&\stackrel{(\text{U})}{\le} \sum^T_{t=1} \sum^\rho_{e=1} \sum_{\myvec{a}^e \in \mathcal{A}^e}P\left(\left|\hat{\mu}^e_{t-1}(\myvec{a}^e) - \mu^e(\myvec{a}^e)\right| > c^e_t(\myvec{a}^e)\right).\\
\end{split}
\end{equation}
The estimated mean $\hat{\mu}^e_{t-1}(\myvec{a}^e)$ is a weighted sum of $n^e_{t-1}(\myvec{a}^e)$ random variables distributed according to a $\sigma$-subgaussian with mean $\mu^e(\myvec{a}^e)$. Hence, Hoeffding's inequality (H) is applicable \cite{vershynin2018high}.
\begin{equation}
\begin{split}
P\left(\left|\hat{\mu}^e_{t-1}(\myvec{a}^e) - \mu^e(\myvec{a}^e)\right| > c^e_t(\myvec{a}^e)\ \middle|\ \mu^e(\myvec{a}^e)\right)
&\stackrel{(\text{H})}{\le} 2\exp\left(-\frac{n^e_{t-1}(\myvec{a}^e)}{2\sigma^2}(c^e_t(\myvec{a}^e))^2\right)\\
&\stackrel{(\ref{eq:threshold})}{=} 2\exp\left(-\frac{n^e_{t-1}(\myvec{a}^e)}{2\sigma^2}\frac{2\sigma^2\log(\delta^{-1})}{n^e_{t-1}(\myvec{a}^e)}\right)\\
&= 2\exp\left(-\log(\delta^{-1})\right).\\
&= 2\delta\\
\end{split}
\end{equation}
Therefore, the following concentration inequality on $\eventc$ holds:
\begin{equation}
\begin{split}
P(\eventc) &\le \sum^T_{t=1} \sum^\rho_{e=1} \sum_{\myvec{a}^e \in \mathcal{A}^e} 2 \delta = 2\locactsize T \delta.\\
\end{split}
\end{equation}
\end{proof}
\begin{lemma}\emph{(Bayesian regret bound under $\event$)} Provided that the error bound on the local sample means is never exceeded until time $T$, the Bayesian regret bound, when using the MATS policy $\pi$, is of the order
\begin{equation}
\begin{split}
\mathbb{E}\left[R(T, \pi)\ \middle|\ \event\right] &\le \sqrt{32 \sigma^2 \locactsize \rho T \log(\delta^{-1})}.
\end{split}
\end{equation}
\label{lemma:regret_event}
\end{lemma}
\begin{proof}
Consider this upper bound on the sample means:
\begin{equation}
\begin{split}
u_t(\myvec{a}) &\triangleq \sum^\rho_{e=1} \hat{\mu}^e_{t-1}(\myvec{a}^e) + c^e_t(\myvec{a}^e).
\label{eq:u_t}
\end{split}
\end{equation}
Given history $\mathcal{H}_{t-1}$, the statistics $\hat{\mu}^e_{t-1}(\myvec{a}^e)$ and $n^e_{t-1}(\myvec{a}^e)$ are known, rendering $u_t(\cdot)$ a deterministic function. Therefore, the probability matching property of MATS (Equation~\ref{eq:prob_match}) can be applied as follows:
\begin{equation}
\begin{split}
&\mathbb{E}\left[u_t(\myvec{a}_t)\ |\ \hist\right] = \mathbb{E}\left[u_t(\myvec{a}_*)\ |\ \hist\right].\\
\label{eq:prob_match_ut}
\end{split}
\end{equation}
Hence, using the tower-rule (T), the regret can be bounded as
\begin{equation}
\begin{split}
\mathbb{E}\left[\sum^T_{t=1} \Delta(\myvec{a}_t)\ |\ \event\right] &\stackrel{(\text{T})}{=} \mathbb{E}\left[\sum^T_{t=1} \mathbb{E}\left[\mu(\myvec{a}_*) - \mu(\myvec{a}_t)\ |\ \hist, \event\right]\right]\\
&= \mathbb{E}\left[\sum^T_{t=1} \mathbb{E}\left[\mu(\myvec{a}_*) - u_t(\myvec{a}_t)\ |\ \hist, \event\right]\right.\\
&\qquad + \left.\sum^T_{t=1}\mathbb{E}\left[u_t(\myvec{a}_t) - \mu(\myvec{a}_t)\ |\ \hist, \event\right]\right]\\
&\stackrel{(\ref{eq:prob_match_ut})}{=} \mathbb{E}\left[\sum^T_{t=1} \mathbb{E}\left[\mu(\myvec{a}_*) - u_t(\myvec{a}_*)\ |\ \hist, \event\right]\right.\\
&\qquad + \left.\sum^T_{t=1}\mathbb{E}\left[u_t(\myvec{a}_t) - \mu(\myvec{a}_t)\ |\ \hist, \event\right]\right].\\
\label{eq:bound_regret_t}
\end{split}
\end{equation}
Note that the expression $\mu(\myvec{a}_*) - u_t(\myvec{a}_*)$ is always negative under $\event$, i.e.,
\begin{equation}
\begin{split}
\mu(\myvec{a}_*) - u_t(\myvec{a}_*) &\stackrel{(\ref{eq:u_t})}{=} \sum^{\rho}_{e=1} \mu^e(\myvec{a}^e_*) - \hat{\mu}^e_{t-1}(\myvec{a}^e_*) - c^e_t(\myvec{a}^e_*)\\
&\stackrel{(\ref{eq:event})}{\le} \sum^{\rho}_{e=1} c^e_t(\myvec{a}^e_*) - c^e_t(\myvec{a}^e_*) = 0,\\
\end{split}
\end{equation}
while $u_t(\myvec{a}_t) - \mu(\myvec{a}_t)$ is bounded by twice the threshold $c^e_t(\myvec{a}^e)$, i.e.,
\begin{equation}
\begin{split}
u_t(\myvec{a}_t) - \mu(\myvec{a}_t)  &\stackrel{(\ref{eq:u_t})}{=} \sum^{\rho}_{e=1} \hat{\mu}^e_{t-1}(\myvec{a}^e_t) + c^e_t(\myvec{a}^e_t) - \mu^e(\myvec{a}^e_t) \\
&\stackrel{(\ref{eq:event})}{\le} \sum^{\rho}_{e=1} c^e_t(\myvec{a}^e_t) + c^e_t(\myvec{a}^e_t) = 2 \sum^{\rho}_{e=1} c^e_t(\myvec{a}^e_t).\\
\end{split}
\end{equation}
Thus, Equation~\ref{eq:bound_regret_t} can be bounded as
\begin{equation}
\begin{split}
\mathbb{E}\left[\sum^T_{t=1} \Delta(\myvec{a}_t)\ |\ \event\right]
&\le 2\sum^T_{t=1} c^e_t(\myvec{a}^e_t)\\
&\le 2\sum^T_{t=1} \sqrt{\frac{2\sigma^2\log(\delta^{-1})}{n^e_{t-1}(\myvec{a}^e_t)}}\\
&= 2 \sum_{\myvec{a}^e \in \mathcal{A}^e} \sum^T_{t=1} \mathcal{I}\{\myvec{a}^e_t = \myvec{a}^e\}\sqrt{\frac{2\sigma^2\log(\delta^{-1})}{n^e_{t-1}(\myvec{a}^e)}},
\label{eq:bound_regret_ct}
\end{split}
\end{equation}
where $\mathcal{I}\{\cdot\}$ is the indicator function.
The terms in the summation are only non-zero at the time steps when the local action $\myvec{a}^e$ is pulled, i.e., when $\mathcal{I}\{\myvec{a}^e_t = \myvec{a}^e\} = 1$. Additionally, note that only at these time steps, the counter $n^e_t(\myvec{a}^e)$ increases by exactly 1. Therefore, the following equality holds:
\begin{equation}
\begin{split}
\sum^T_{t=1} \mathcal{I}\{\myvec{a}^e_t = \myvec{a}^e\} \sqrt{(n^e_{t-1}(\myvec{a}^e))^{-1}} = \sum^{n^e_{T}(\myvec{a}^e)}_{k=1} \sqrt{k^{-1}}.\\
\label{eq:equality_counts}
\end{split}
\end{equation}
The function $\sqrt{k^{-1}}$ is decreasing and integrable. Hence, using the right Riemann sum,
\begin{equation}
\sqrt{k^{-1}} \le \int^{k}_{k - 1} \sqrt{x^{-1}} dx.
\label{eq:integrable_sqrt}
\end{equation}
Combining Equations~\ref{eq:bound_regret_ct}-\ref{eq:integrable_sqrt} leads to a bound
\begin{equation}
\begin{split}
\mathbb{E}\left[\sum^T_{t=1} \Delta(\myvec{a}_t)\ \middle|\ \event\right] &\stackrel{(\ref{eq:bound_regret_ct})}{=} 2 \sum_{\myvec{a}^e \in \mathcal{A}^e} \sum^T_{t=1} \mathcal{I}\{\myvec{a}^e_t = \myvec{a}^e\}\sqrt{\frac{2\sigma^2\log(\delta^{-1})}{n^e_{t-1}(\myvec{a}^e)}}\\
&\stackrel{(\ref{eq:equality_counts})}{=} \sqrt{8\sigma^2\log(\delta^{-1})} \sum_{\myvec{a}^e \in \mathcal{A}^e} \sum^{n^e_{T}(\myvec{a}^e)}_{k=1} \sqrt{k^{-1}}\\
&\stackrel{(\ref{eq:integrable_sqrt})}{\le} \sqrt{8\sigma^2\log(\delta^{-1})} \sum_{\myvec{a}^e \in \mathcal{A}^e} \int^{n^e_{T}(\myvec{a}^e)}_{0} \sqrt{x^{-1}} dx\\
&= \sqrt{8\sigma^2\log(\delta^{-1})} \sum_{\myvec{a}^e \in \mathcal{A}^e} \sqrt{4 n^e_{T}(\myvec{a}^e)}.\\
\label{eq:bound_ct_sum_full}
\end{split}
\end{equation}
We use the relationship $||\mathbf{x}||_1 \le \sqrt{n}||\mathbf{x}||_2$ between the 1- and 2-norm of a vector $\mathbf{x}$, where $n$ is the number of elements in the vector, as follows:
\begin{equation}
\begin{split}
\sum^\rho_{e=1} \sum_{\myvec{a}^e \in \mathcal{A}^e} \left|\sqrt{n^e_T(\mathbf{a}^e)}\right| \le \sqrt{\locactsize}\sqrt{\sum^\rho_{e=1} \sum_{\myvec{a}^e \in \mathcal{A}^e} \left(\sqrt{n^e_T(\mathbf{a}^e)}\right)^2}.
\label{eq:norm_inequality}
\end{split}
\end{equation}
Finally, note that the sum of all counts $n^e_T(\mathbf{a}^e)$ is equal to the total number of local pulls done by MATS until time $T$, i.e.,
\begin{equation}
\begin{split}
\sum^{\rho}_{e=1} \sum_{\myvec{a}^e \in \mathcal{A}^e} n^e_{T}(\myvec{a}^e) = \rho T.
\label{eq:total_count}
\end{split}
\end{equation}
Using the Equations~\ref{eq:bound_ct_sum_full}-\ref{eq:total_count}, the complete regret bound under $\event$ is given by
\begin{equation}
\begin{split}
\mathbb{E}\left[\sum^T_{t=1} \Delta(\myvec{a}_t)\ |\ \event\right] &\stackrel{(\ref{eq:bound_ct_sum_full})}{\le} \sqrt{8\sigma^2\log(\delta^{-1})} \sum_{\myvec{a}^e \in \mathcal{A}^e} \sqrt{4 n^e_{T}(\myvec{a}^e)}\\
&\stackrel{(\ref{eq:norm_inequality})}{\le}
\sqrt{32\sigma^2\log(\delta^{-1})} \sqrt{\locactsize}\sqrt{\sum^\rho_{e=1} \sum_{\myvec{a}^e \in \mathcal{A}^e} \left(\sqrt{n^e_T(\mathbf{a}^e)}\right)^2}\\
&\stackrel{(\ref{eq:total_count})}{=} \sqrt{32\sigma^2\log(\delta^{-1})} \sqrt{\locactsize}\sqrt{\rho T}.\\
\end{split}
\end{equation}
\end{proof}

\begin{theorem}
Let $\langle\mathcal{D}, \mathcal{A}, \mathcal{F}\rangle$ be a MAMAB. If Assumptions~\ref{assumption:means} and~\ref{assumption:rewards} hold, then the MATS policy $\pi$ satisfies a Bayesian regret bound of
\begin{equation}
\begin{split}
\E{R(T, \pi)} &\le \sqrt{64\sigma^2 \locactsize \rho T \log(\locactsize T)} + \frac{2}{\locactsize}\\
&\in O\left(\sqrt{\sigma^2 \locactsize \rho T \log(\locactsize T)}\right).
\end{split}
\end{equation}
\label{theorem:regret}
\end{theorem}
\begin{proof}
Using the law of excluded middle (M) and the fact that $\Delta(\myvec{a}_t)$ and $P(\event\ |\ \hist)$ are between 0 and 1 (B), the regret can be decomposed as
\begin{equation}
\begin{split}
\mathbb{E}\left[\sum^T_{t=1} \Delta(\myvec{a}_t)\right] &\stackrel{(\text{M})}{=} \mathbb{E}\left[\sum^T_{t=1} \Delta(\myvec{a}_t)\ |\ \event\right]P(\event) + \mathbb{E}\left[\sum^T_{t=1} \Delta(\myvec{a}_t)\ |\ \eventc\right]P(\eventc)\\
&\stackrel{(\text{B})}{\le} \E{\sum^T_{t=1} \Delta(\myvec{a}_t)\ |\ \event} + T P(\eventc).\\
\end{split}
\label{eq:decomp_bound}
\end{equation}
Then, according to Lemmas~\ref{lemma:conc_ineq} and~\ref{lemma:regret_event} (L), we have
\begin{equation}
\begin{split}
\E{\sum^T_{t=1} \Delta(\myvec{a}_t)} &\stackrel{(\ref{eq:decomp_bound})}{\le} \E{\sum^T_{t=1} \Delta(\myvec{a}_t)\ |\ \event} + T P(\eventc)\\
&\stackrel{(\text{L})}{\le} \sqrt{32 \sigma^2 \locactsize \rho T \log(\delta^{-1})} + 2 \locactsize T^2 \delta.\\
\end{split}
\label{eq:combine_lemmas}
\end{equation}
Finally, choosing $\delta = (\locactsize T)^{-2}$, we conclude that
\begin{equation}
\begin{split}
\E{R(T, \pi)} &\stackrel{(\ref{eq:combine_lemmas})}{\le} \sqrt{32 \sigma^2 \locactsize \rho T \log(\delta^{-1})} + 2 \locactsize T^2 \delta\\
 &\le \sqrt{64 \sigma^2 \locactsize \rho T \log\left(\locactsize T\right)} + \frac{2}{\locactsize}\\
&\in O\left(\sqrt{\sigma^2 \locactsize \rho T\log(\locactsize T)}\right).\\
\end{split}
\end{equation}
\end{proof}

\begin{corollary} If $|\mathcal{A}_i| \le k$ for all agents $i$, and if $|\mathcal{D}^e| \le d$ for all groups $\mathcal{D}^e$, then
\begin{equation}
\begin{split}
\E{R(T, \pi)} \in O\left(\rho\sqrt{\sigma^2k^d T \log(\rho k^d T)}\right).
\end{split}
\end{equation}
 \label{cor:regret}
\end{corollary}
\begin{proof}
$\locactsize = \sum_{e=1}^\rho |\mathcal{A}^e| = \sum_{e=1}^\rho \prod_{i \in \mathcal{D}^e} |\mathcal{A}_i| \le \rho k^d$.
\end{proof}

Corollary~\ref{cor:regret} tells us that the regret is sub-linear in terms of time $T$ and low-order polynomial in terms of the largest action space of a single agent when the number of groups and agents per group are small. This reflects the main contribution of this work. When agents are loosely coupled, the \emph{effective} joint arm space is significantly reduced, and MATS provides a mechanism that efficiently deals with such settings. This is a significant improvement over the established classic regret bounds of vanilla TS when the MAMAB is `flattened' and the factored structure is neglected \cite{russo2014learning,lattimore2018bandit}. The classic bounds scale exponentially with the number of agents, which renders the use of vanilla TS unfeasible in many multi-agent environments.

\section*{Experiments}
\label{sec:experiments}

We evaluate the performance of MATS on the benchmark problems proposed in the paper that introduced MAUCE \cite{bargiacchi2018learning}, which is the current state-of-the-art algorithm for multi-agent bandit problems, and one novel setting that falls outside the domain of the theoretical guarantees for both MAUCE and MATS.
First, we evaluate the performance of MATS on two benchmarks that were introduced in the MAUCE paper, i.e., Bernoulli 0101-Chain and Gem Mining. We compare against a random policy (rnd), Sparse Cooperative Q-Learning (SCQL) \cite{kok2004scql} and the state-of-the-art algorithm, MAUCE \cite{bargiacchi2018learning}. For SCQL and MAUCE, we use the same exploration parameters as in previous work \cite{bargiacchi2018learning}. For MATS, we always use non-informative Jeffreys priors, which are invariant toward reparametrization of the experimental settings \cite{robert2007bayesian}. Although including additional prior domain knowledge could be useful in practice, we use well-known non-informative priors in our experiments to compare fairly with other state-of-the-art techniques. 
Then, we introduce a novel variant of the 0101-Chain with Poisson-distributed local rewards. A Poisson distribution is supergaussian, meaning that its tails tend slower towards zero than the tails of any Gaussian. Therefore, both the assumptions made in Theorem~\ref{theorem:regret} and in the established regret bound of MAUCE are violated. Additionally, as the rewards are highly skewed, we expect that the use of symmetric exploration bounds in MAUCE will often lead to either over- or underexploration of the local arms. We assess the performance of both methods on this benchmark.



\subsection*{Bernoulli 0101-Chain}
\label{sec:exp_chains}


The Bernoulli 0101-Chain consists of $n$ agents and $n-1$ local reward distributions. Each agent can choose between two actions: 0 and 1. In the coordination graph, agents $i$ and $i+1$ are connected to a local reward $f^i(a_i, a_{i+1})$. Thus, each pair of agents should locally coordinate in order to find the best joint arm. The local rewards are drawn from a Bernoulli distribution with a different success probability per group. These success probabilities are given in Table~\ref{tab:ber_chain}. The optimal joint action is an alternating sequence of zeros and ones, starting with 0.

\begin{table}[!ht]
\centering
\begin{tabular}{| l | c | c |}
  \hline \textbf{$f^i \sim \mathcal{B}$} & $a_{i+1} = 0$ & $a_{i+1} = 1$\\
  \hline $a_{i} = 0$ & $0.75$ & $1$\\
  \hline $a_{i} = 1$ & $0.25$ & $0.9$\\ 
  \hline
\end{tabular}
\vspace{0.3cm}
\caption{
Bernouilli 0101 Chain -- The unscaled local reward distributions of agents $i$ and $i+1$, where $i$ is even. Each entry shows the success probability for each local arm of agents $i$ and $i+1$, where $i$ is even. The table is transposed for the case where $i$ is odd.}
\label{tab:ber_chain}
\end{table}

To ensure that the assumptions made in the regret analyses of MAUCE and MATS hold, we divide the local rewards by the number of groups, such that the global rewards are between 0 and 1. 
\begin{figure*}[ht!]
\centering
\subfigure[Bernoulli 0101-Chain]{\label{fig:nodes}\includegraphics[width=0.49\textwidth]{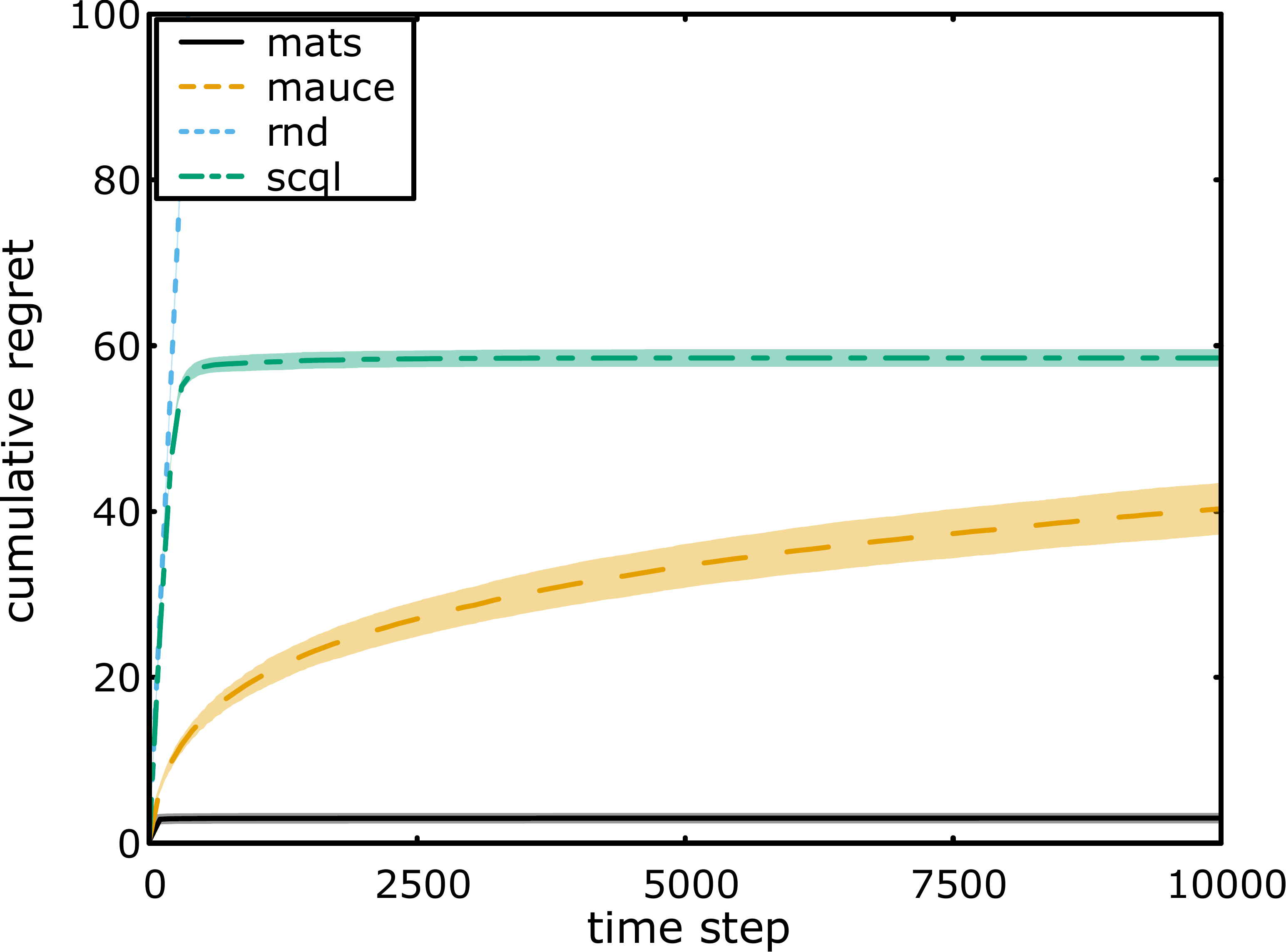}}
\subfigure[Gem Mining]{\label{fig:mines}\includegraphics[width=0.49\textwidth]{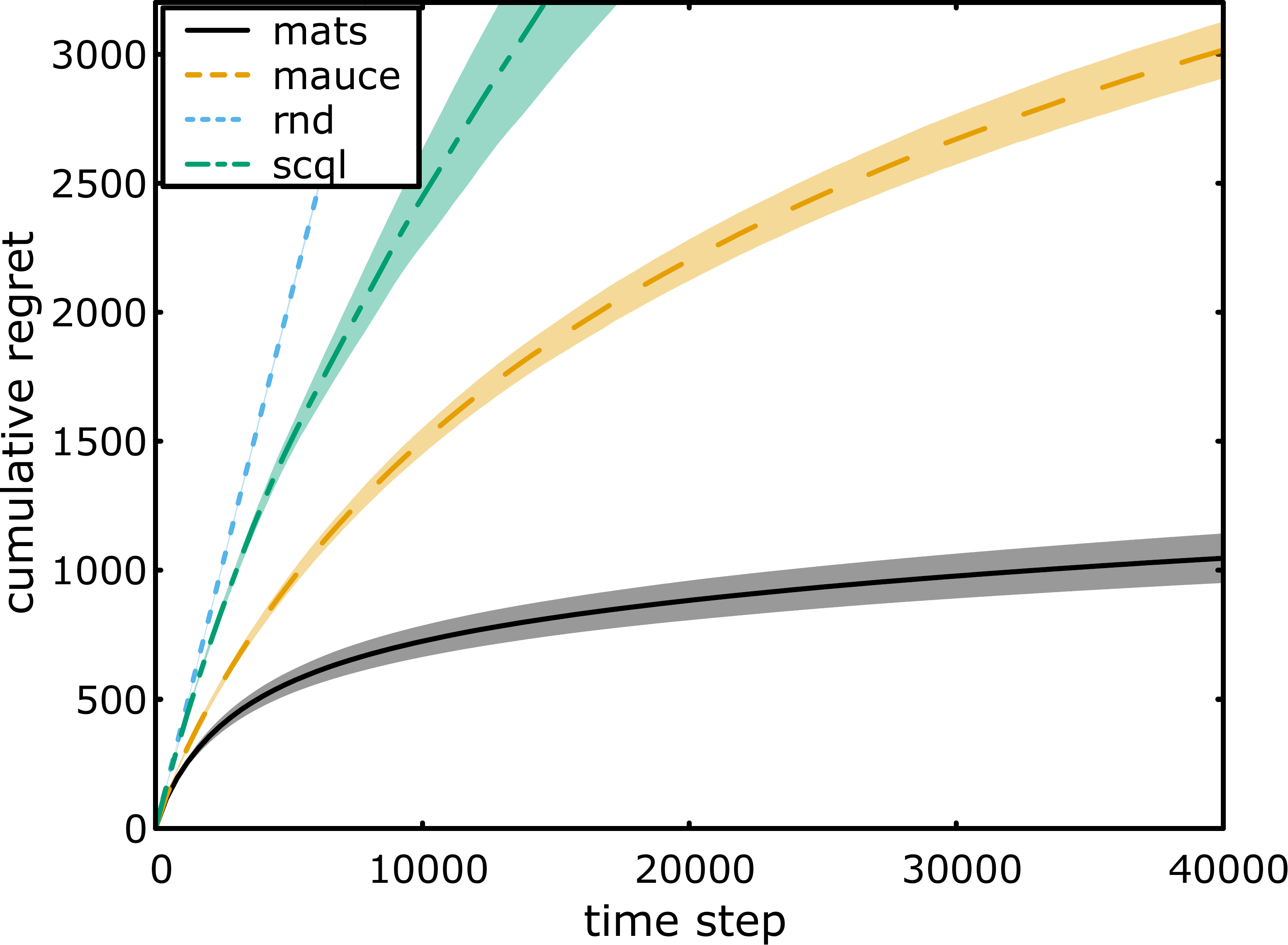}}\\
\subfigure[Poisson 0101-Chain]{\label{fig:nodesp}\includegraphics[width=0.49\textwidth]{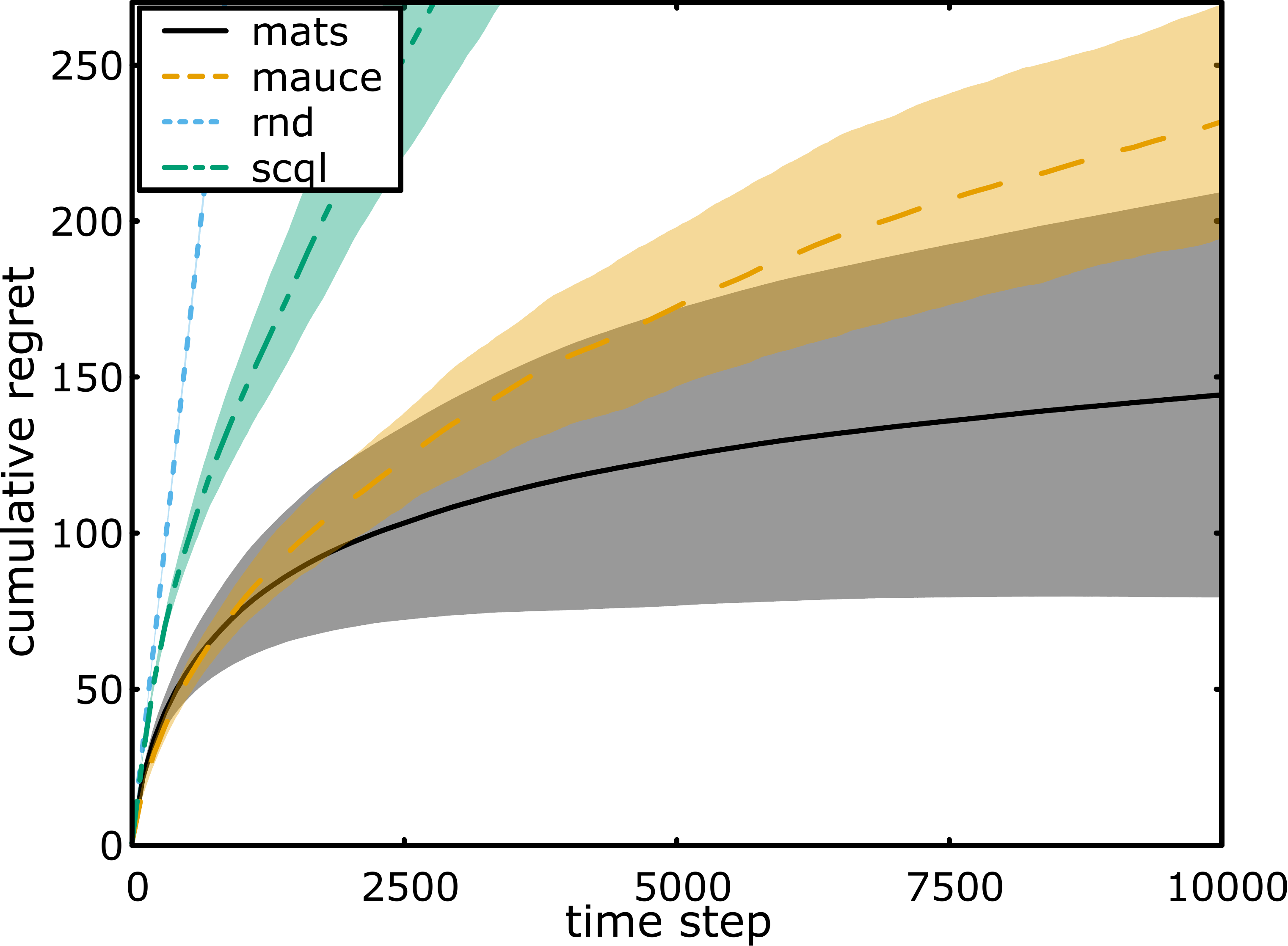}}
\caption{Cumulative normalized regret averaged over 100 runs for the (a) Bernoulli 0101-Chain, (b) Gem Mining and (d) Poisson 0101-Chain, and over 10 runs for the (c) Wind Farm. Both the mean (line) and standard deviation (shaded area) are plotted.}
\label{fig:exp_results}
\end{figure*}

We provide non-informative Jeffreys priors on the unknown means to MATS, which for the Bernoulli likelihood is a Beta prior, $\mathcal{B}(\alpha=0.5, \beta=0.5)$ \cite{lunn2012bugs}.
The results for the Bernoulli 0101-chains are shown in Figure~\ref{fig:nodes}.

\subsection*{Gem Mining}

\begin{figure}[!ht]
\centering
\includegraphics[width=0.7\textwidth]{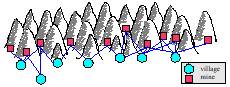}
\caption{Example of a coordination graph in the Gem Mining problem. The red nodes are the mines (rewards), while the blue nodes are the villages (agents).}
\label{fig:mines_graph}
\end{figure}

In the Gem Mining problem, a mining company wants to excavate a set of mines for gems (i.e., local rewards). The goal is to maximize the total number of gems found over all mines. However, the company's workers live in separate villages (i.e., agents), and only one van per village is available. Therefore, each village needs to decide to which mine it should send its workers (i.e., local action). Moreover, workers can only commute to nearby mines (i.e., coordination graph). Hence, a group can be constructed per mine, consisting of all agents that can travel toward the mine. An example of a coordination graph is given in Figure~\ref{fig:mines_graph}

The reward is drawn from a Bernoulli distribution, where the probability of finding a gem at a mine is $1.03^{w-1} p$ with $w$ the number of workers at the mine and $p$ a base probability that is sampled uniformly random from the interval $[0, 0.5]$ for each mine.
When more workers are excavating a mine, the probability of finding a gem increases. Each village is populated by a number sampled uniformly random from $[1..5]$. The coordination graph is generated by sampling for each village $i$ a number of mines $m_i$ in $[2..4]$ to which it should be connected. Then, each village $i$ is connected to the mines $i$ to $(i+m_i-1)$. The last village is always connected to 4 mines.

We provide non-informative Jeffreys priors on the unknown means to MATS, which for the Bernoulli likelihood is a Beta prior, $\mathcal{B}(\alpha=0.5, \beta=0.5)$ \cite{lunn2012bugs}.
The results for the Gem Mining problem are shown in Figure~\ref{fig:mines}.

\subsection*{Poisson 0101-Chain}

We introduce a novel benchmark with Poisson distributed local rewards, for which the established regret bounds of MATS and MAUCE do not hold.
Similar to the Bernoulli 0101-Chain, agents need to coordinate their actions in order to obtain an alternating sequence of zeroes and ones. However, as the rewards are highly skewed and supergaussian, this setting is much more challenging. The means of the Poisson distributions are given in Table~\ref{tab:poisson_chain}. We also divide the rewards by the number of groups, similar to the Bernoulli 0101-Chain.

\begin{table}[!ht]
\centering
\begin{tabular}{| l | c | c |}
  \hline \textbf{$f^i \sim \mathcal{P}$} & $a_{i+1} = 0$ & $a_{i+1} = 1$\\
  \hline $a_{i} = 0$ & $0.1$ & $0.3$\\
  \hline $a_{i} = 1$ & $0.2$ & $0.1$\\ 
  \hline
\end{tabular}
\vspace{0.3cm}
\caption{
Poisson 0101 Chain -- The unscaled local reward distributions of agents $i$ and $i+1$. Each entry shows the mean for each local arm of agents $i$ and $i+1$.}
\label{tab:poisson_chain}
\end{table}

For MAUCE, an exploration parameter must be chosen. This exploration parameter denotes the range of the observed rewards. As a Poisson distribution has unbounded support, we rely on percentiles of the reward distribution. Specifically, as 95$\%$ of the rewards when pulling the optimal arm falls below $1$, we choose $1$ as the exploration parameter of MAUCE.
For MATS we use non-informative Jeffreys priors on the unknown means, which for the Poisson likelihood is a Gamma prior, $\mathcal{G}(\alpha=0.5, \beta=0)$ \cite{lunn2012bugs}.
The results are shown in Figure~\ref{fig:nodesp}.

\section*{Wind farm control application}
\label{sec:experiments_wind}
We demonstrate the benefits of MATS on a state-of-the-art wind farm simulator and compare its performance to MAUCE and SCQL.
A wind farm consists of a group of wind turbines, instantiated to extract energy from wind. From the perspective of a single turbine, aligning with the incoming wind vector usually ensures the highest productivity. However, translating this control policy directly towards an entire wind farm may be sub-optimal. As wind passes through the farm, downstream turbines observe a significantly lower wind speed. This is known as the \emph{wake effect}, which is due to the turbulence generated behind operational turbines.

In recent work, the possibility of deflecting wake away from the farm through rotor misalignment is investigated \cite{vandijk2016}. While a misaligned turbine produces less energy on its own, the group's total productivity is increased.
Physically, the wake effect reduces over long distances, and thus, turbines tend to only influence their neighbours. We can use this domain knowledge to define groups of agents and organize them in a graph structure. Note that the graph structure depends on the incoming wind vector. Nevertheless, atmospheric conditions are typically discretized when analyzing operational regimes \cite{iso_standard_design_2012}, thus, a graph structure can be made independently for each possible incoming discretized wind vector. We construct a graph structure for one possible wind vector.

\begin{figure}[!ht]
\centering
\subfigure[Dependency graph]{\label{fig:wind_graph}\includegraphics[width=0.49\textwidth, trim=3cm 0cm 7cm 1.5cm, clip]{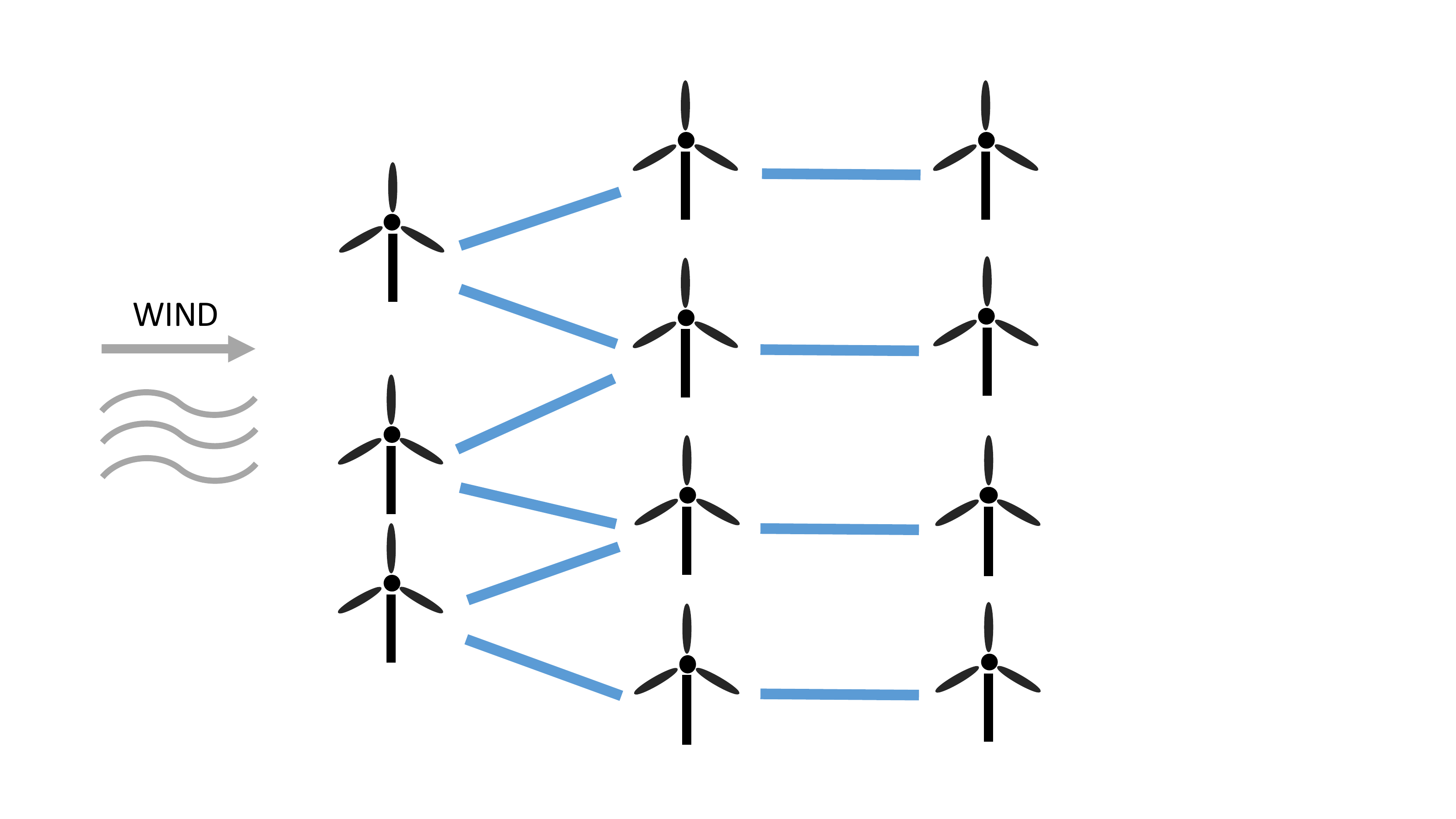}}
\subfigure[Wind Farm]{\label{fig:wind}\includegraphics[width=0.49\textwidth]{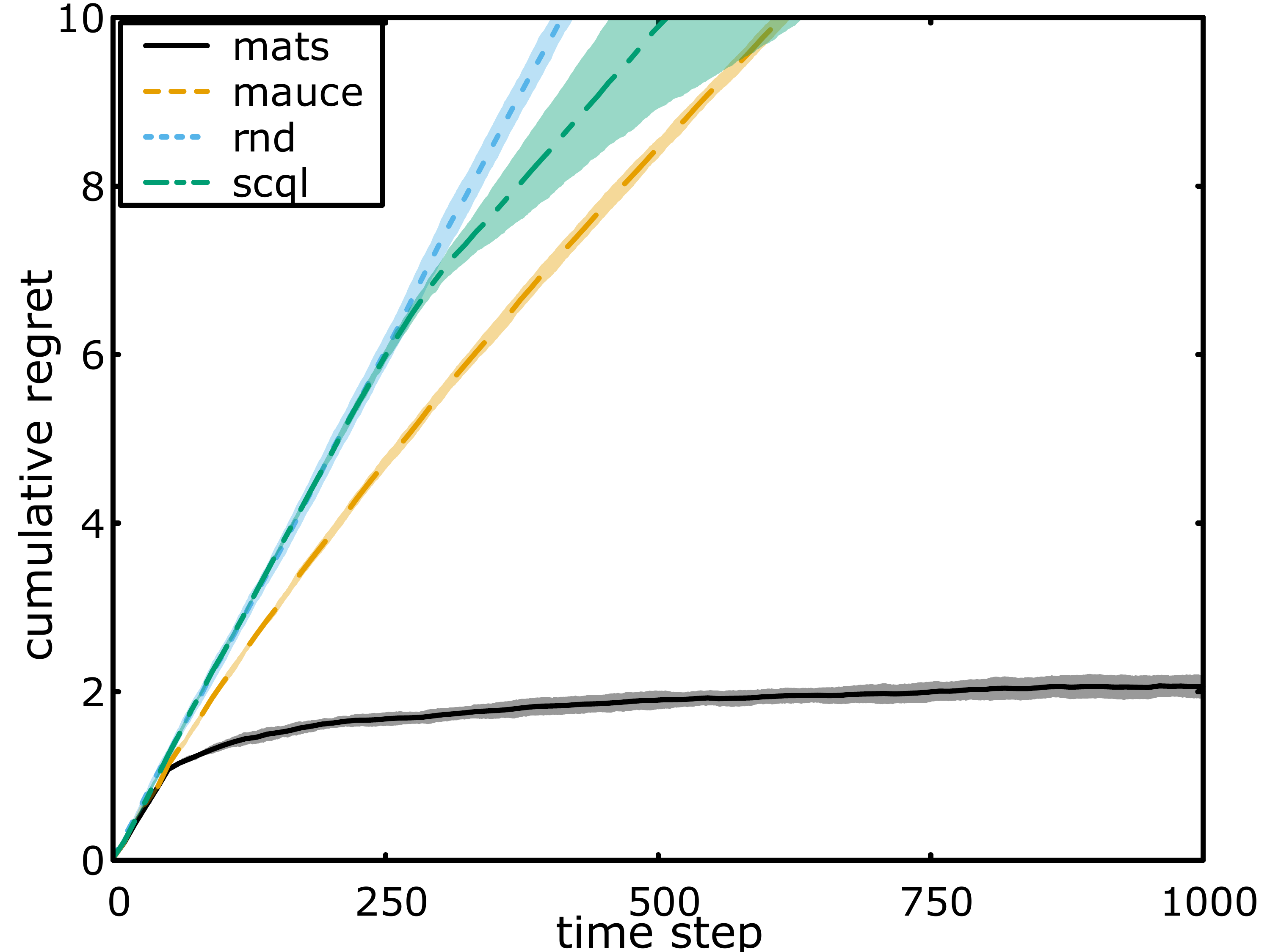}}
\caption{Wind farm layout -- Dependency graph where the nodes are the turbines and the edges describe the dependencies between the turbines. The incoming wind is denoted by an arrow.}
\end{figure}

We demonstrate our method on a virtual wind farm, consisting of 11 turbines, of which the layout is shown in Figure~\ref{fig:wind_graph}. We use the state-of-the-art WISDEM FLORIS simulator \cite{floris_2019}.
For MATS, we assume the local power productions are sampled from Gaussians with unknown mean and variance, which leads to a Student's t-distribution on the mean when using a Jeffreys prior \cite{honda2014optimality}. The results for the wind farm control setting are shown in Figure~\ref{fig:wind}.

\section*{Discussion}
\label{sec:discussion}

MATS is a Bayesian method, which means that it can leverage prior knowledge about the data distribution. This property is highly beneficial in many practical applications, e.g., influenza mitigation \cite{libin2018bayesian} and wind farm control \cite{verstraeten2019fleetwide, bfts2019}.

Both MAUCE and MATS achieve sub-linear regret in terms of time and low-order polynomial regret in terms of the number of local arms for sparse coordination graphs. However, empirically, MATS consistently outperforms MAUCE as well as SCQL. 
We can see that MATS solves the Bernoulli 0101-Chain problem in only a few time steps, while MAUCE still pulls many sub-optimal actions after 10000 time steps (see Figure~\ref{fig:nodes}). In the more challenging Gem Mining problem, the cumulative regret of MAUCE is three times as high as the cumulative regret of MATS around 40000 time steps (see Figure~\ref{fig:mines}). In the wind farm control task, we can see that MATS allowed for a five-fold increase of the normalized power productions with respect to the state of the art (see Figure~\ref{fig:wind}).
We argue that the high performance of MATS is due to the ability to seamlessly include domain knowledge about the shape of the reward distributions and treat the problem parameters as unknowns. To highlight the power of this property, we introduced the Poisson 0101-chain. In this setting, the reward distributions are highly skewed, for which the mean does not match the median. Therefore, in our case, since the mean falls well above 50\% of all samples, it is expected that for the initially observed rewards, the true mean will be higher than the sample mean. Naturally, this bias averages out in the limit, but may have a large impact during the early exploration stage. The high standard deviations in Figure~\ref{fig:nodesp} support this impact. 
Although the established regret bounds of MATS and MAUCE do not apply for supergaussian reward distributions, we demonstrate that MATS exploits density information of the rewards to achieve more targeted exploration. In Figure~\ref{fig:nodesp}, the cumulative regret of MATS stagnates around 7500 time steps, while the cumulative regret of MAUCE continues to increase significantly. As MAUCE only supports symmetric exploration bounds, it is challenging to correctly assess the amount of exploration needed to solve the task. 

Throughout the experiments, exploration constants had to be specified for MAUCE, which were challenging to choose and interpret in terms of the density of the data. In contrast, MATS uses either statistics about the data (if available) or, potentially non-informative, beliefs defined by the user. For example, in the wind farm case, the spread of the data is unknown. MATS effectively maintains a posterior on the variance and uses it to balance exploration and exploitation, while still outperforming MAUCE with a manually calibrated exploration range (see Figure~\ref{fig:wind}).

\section*{Related work}
\label{sec:related_work}

Multi-agent reinforcement learning and planning with loose couplings has been investigated in sequential decision problems \cite{guestrin2002context,KokVlassisRobo,DeHauwere2010,scharpff2016solving}. In sequential settings, the value function cannot be factorized exactly. Therefore, it is challenging to provide convergence and optimality guarantees. While for planning some theoretical guarantees can be provided \cite{scharpff2016solving}, in the learning literature the focus has been on empirical validation \cite{KokVlassisRobo}. In this work, we focus on MAMABs, which are single-shot stateless problems. In such settings, the reward function is factored exactly into components that only depend on a subset of agents.

The combinatorial bandit \cite{bubeck2012regret,cesa2012combinatorial,gai2012combinatorial,chen2013combinatorial} is a variant of the multi-armed bandit, in which, rather than one-dimensional arms, an arm vector has to be pulled. In our work, the arms' dimensionality corresponds to the number of agents in our system, and similarly to combinatorial bandits, the number of arms exponentially increases with this quantity. We consider a variant of this framework, called the semi-bandit problem \cite{audibert2011minimax}, in which local components of the global reward are observable. Chen et. al (2013) constructed an algorithm for this setting that assumes access to an $(\alpha, \beta)$-oracle, which provides a joint action that outputs a fraction $\alpha$ of the optimal expected reward with probability $\beta$. Instead, we assume the availability of a coordination graph, which we argue is a reasonable assumption in many multi-agent settings.

Sparse cooperative Q-learning is an algorithm that also assumes the availability of a coordination graph \cite{kok2004scql}. However, although strong experimental results are given, no theoretical guarantees were provided. Later, the UCB-like algorithm, HEIST, for exploration and exploitation in MAMABs was introduced \cite{stranders2012dcops}, which uses a message-passing scheme for resolving coordination graphs. They provide some theoretical guarantees on the regret for problems with acyclic coordination graphs. Multi-Agent Upper-Confidence Exploration (MAUCE) \cite{bargiacchi2018learning} is a more general method that uses variable elimination to resolve (potentially cyclic) coordination graphs. MAUCE demonstrates high performance on a variety of benchmarks and provides a tight theoretical upper bound on the regret. MATS provides a Bayesian alternative to MAUCE based on Thompson sampling (TS).

Our problem definition is related to distributed constraint optimization (DCOP) problems \cite{yokoo1998distributed}. In DCOP problems, multiple agents control a set of variables in a distributed manner under a set of constraints. The objective is the same as for a MAMAB, i.e., optimize the sum over group rewards. However, in DCOPs, the rewards are assumed to be known beforehand. The Distributed Coordination of Exploration and Exploitation (DCEE) framework \cite{taylor2011distributed} extends this setting to unknown rewards, but considers the optimization of the cumulative reward achieved over a time span, rather than of a single-step reward. MAMABs, or MAB-DCOPs \cite{stranders2012dcops}, consider the optimization of a single-step expected reward over time.

In recent research on wind farm control, the impact of optimized rotor alignments on power production is heavily investigated \cite{vandijk2016}. To search for the optimal alignments within the wind farm, data-driven methods are usually adopted, where the turbines' alignments are perturbed iteratively until they locally converge \cite{marden2013model}. When optimizing the alignment of a wind turbine, only considering its neighbours can significantly boost the learning speed \cite{gebraad2015maximum}. MATS is also able to leverage neighbourhood structures. In addition, rather than random perturbation of the alignments, MATS leverages an exploration-exploitation mechanism that is inspired by TS and variable elimination, which allows for a global exploration mechanism that targets the optimal alignment configuration, while retaining a small regret during the learning process itself.


\section*{Conclusions}
\label{sec:conclusions}

We proposed multi-agent Thompson sampling (MATS), a novel Bayesian algorithm for multi-agent multi-armed bandits. The method exploits loose connections between agents to solve multi-agent coordination tasks efficiently. Specifically, we proved that, for $\sigma$-subgaussian rewards with bounded means, the expected cumulative regret decreases sub-linearly in time and low-order polynomially in the highest number of actions of a single agent when the coordination graph is sparse. Empirically, we showed a significant improvement over the state-of-the-art algorithm, MAUCE, on several synthetic benchmarks. Additionally, we showed that MATS can seamlessly be adapted to the available prior knowledge, and achieves state-of-the-art performance on the Poisson 0101-Chain, a new benchmark with supergaussian rewards. Finally, we demonstrated that MATS achieves high performance on a realistic wind farm control task, where the optimal rotor alignments of the wind turbines need to be jointly optimized to maximize the farm's power production. In many practical applications, there exist sparse neighbourhood structures between agents, and we have shown that MATS is able to successfully exploit these structures, while leveraging prior knowledge about the data. 

\bibliographystyle{abbrv}  
\bibliography{refs}  

\begin{thebibliography}{10}

\bibitem{agrawal2012analysis}
S.~Agrawal and N.~Goyal.
\newblock Analysis of thompson sampling for the multi-armed bandit problem.
\newblock In {\em Conference on Learning Theory}, pages 39--1, 2012.

\bibitem{agrawal2013further}
S.~Agrawal and N.~Goyal.
\newblock Further optimal regret bounds for thompson sampling.
\newblock In {\em Artificial intelligence and statistics}, pages 99--107, 2013.

\bibitem{audibert2011minimax}
J.-Y. Audibert, S.~Bubeck, and G.~Lugosi.
\newblock Minimax policies for combinatorial prediction games.
\newblock In {\em COLT}, volume~19, pages 107--132, 2011.

\bibitem{bargiacchi2018learning}
E.~Bargiacchi, T.~Verstraeten, D.~M. Roijers, A.~Now{\'e}, and H.~Hasselt.
\newblock Learning to coordinate with coordination graphs in repeated
  single-stage multi-agent decision problems.
\newblock In {\em International Conference on Machine Learning}, pages
  491--499, 2018.

\bibitem{Boutilier1996mmdp}
C.~Boutilier.
\newblock {Planning, learning and coordination in multiagent decision
  processes}.
\newblock In {\em TARK 1996: Proceedings of the 6th conference on Theoretical
  aspects of rationality and knowledge}, pages 195--210, 1996.

\bibitem{bubeck2012regret}
S.~Bubeck and N.~Cesa-Bianchi.
\newblock Regret analysis of stochastic and nonstochastic multi-armed bandit
  problems.
\newblock {\em Foundations and Trends in Machine Learning}, 5(1):1--122, 2012.

\bibitem{cesa2012combinatorial}
N.~Cesa-Bianchi and G.~Lugosi.
\newblock Combinatorial bandits.
\newblock {\em Journal of Computer and System Sciences}, 78(5):1404--1422,
  2012.

\bibitem{chapelle2011empirical}
O.~Chapelle and L.~Li.
\newblock An empirical evaluation of thompson sampling.
\newblock In {\em Advances in neural information processing systems}, pages
  2249--2257, 2011.

\bibitem{chapman2013convergent}
A.~C. Chapman, D.~S. Leslie, A.~Rogers, and N.~R. Jennings.
\newblock Convergent learning algorithms for unknown reward games.
\newblock {\em SIAM Journal on Control and Optimization}, 51(4):3154--3180,
  2013.

\bibitem{chen2013combinatorial}
W.~Chen, Y.~Wang, and Y.~Yuan.
\newblock Combinatorial multi-armed bandit: General framework, results and
  applications.
\newblock In {\em Proceedings of the 30th international conference on machine
  learning}, pages 151--159, 2013.

\bibitem{claes2017decentralised}
D.~Claes, F.~Oliehoek, H.~Baier, and K.~Tuyls.
\newblock Decentralised online planning for multi-robot warehouse
  commissioning.
\newblock In {\em Proceedings of the 16th Conference on Autonomous Agents and
  MultiAgent Systems}, pages 492--500. International Foundation for Autonomous
  Agents and Multiagent Systems, 2017.

\bibitem{DeHauwere2010}
Y.-M. De~Hauwere, P.~Vrancx, and A.~Now{\'e}.
\newblock Learning multi-agent state space representations.
\newblock In {\em Proceedings of the 9th International Conference on Autonomous
  Agents and Multiagent Systems}, AAMAS '10, pages 715--722, 2010.

\bibitem{gai2012combinatorial}
Y.~Gai, B.~Krishnamachari, and R.~Jain.
\newblock Combinatorial network optimization with unknown variables:
  Multi-armed bandits with linear rewards and individual observations.
\newblock {\em IEEE/ACM Transactions on Networking (TON)}, 20(5):1466--1478,
  2012.

\bibitem{gebraad2015maximum}
P.~M. Gebraad and J.-W. van Wingerden.
\newblock Maximum power-point tracking control for wind farms.
\newblock {\em Wind Energy}, 18(3):429--447, 2015.

\bibitem{VIFMDP}
C.~Guestrin, D.~Koller, and R.~Parr.
\newblock Max-norm projections for factored {MDP}s.
\newblock In {\em Proc. of the 17th International Joint Conference on
  Artificial Intelligence (IJCAI)}, pages 673--682, 2001.

\bibitem{guestrin2002multiagent}
C.~Guestrin, D.~Koller, and R.~Parr.
\newblock Multiagent planning with factored mdps.
\newblock In {\em Advances in neural information processing systems}, pages
  1523--1530, 2002.

\bibitem{guestrin2002context}
C.~Guestrin, S.~Venkataraman, and D.~Koller.
\newblock Context-specific multiagent coordination and planning with factored
  mdps.
\newblock In {\em AAAI/IAAI}, pages 253--259, 2002.

\bibitem{honda2014optimality}
J.~Honda and A.~Takemura.
\newblock Optimality of thompson sampling for gaussian bandits depends on
  priors.
\newblock In {\em Artificial Intelligence and Statistics}, pages 375--383,
  2014.

\bibitem{iso_standard_design_2012}
{International Electrotechnical Commission}.
\newblock {W}ind turbines -- {Part 4: D}esign requirements for wind turbine
  gearboxes ({No. IEC 61400-4}), 2012.
\newblock accessed 6 March 2019.

\bibitem{KokVlassisRobo}
J.~Kok and N.~Vlassis.
\newblock Using the max-plus algorithm for multiagent decision making in
  coordination graphs.
\newblock In A.~Bredenfeld, A.~Jacoff, I.~Noda, and Y.~Takahashi, editors, {\em
  RoboCup 2005: Robot Soccer World Cup IX}, volume 4020 of {\em Lecture Notes
  in Computer Science}, pages 1--12. Springer, 2006.

\bibitem{kok2004scql}
J.~R. Kok and N.~Vlassis.
\newblock Sparse cooperative q-learning.
\newblock In {\em Proceedings of the Twenty-first International Conference on
  Machine Learning}, ICML '04, New York, NY, USA, 2004.

\bibitem{PIFMDP}
D.~Koller and R.~Parr.
\newblock Policy iteration for factored {MDP}s.
\newblock In {\em Proceedings of the Sixteenth Conference on Uncertainty in
  Artificial Intelligence (UAI)}, pages 326--334, San Francisco, CA, USA, 2000.
  Morgan Kaufmann Publishers Inc.

\bibitem{lattimore2018bandit}
T.~Lattimore and C.~Szepesv{\'a}ri.
\newblock Bandit algorithms.
\newblock {\em preprint}, 2018.

\bibitem{bfts2019}
P.~Libin, T.~Verstraeten, D.~M. Roijers, W.~Wang, K.~Theys, and A.~Nowé.
\newblock Thompson sampling for m-top exploration.
\newblock In {\em Proceedings of the IEEE 31st International Conference on
  Tools with Artificial Intelligence (ICTAI)}, pages 1414--1420, 2019.

\bibitem{libin2018bayesian}
P.~J. Libin, T.~Verstraeten, D.~M. Roijers, J.~Grujic, K.~Theys, P.~Lemey, and
  A.~Now{\'e}.
\newblock Bayesian best-arm identification for selecting influenza mitigation
  strategies.
\newblock In {\em Joint European Conference on Machine Learning and Knowledge
  Discovery in Databases}, pages 456--471. Springer, 2018.

\bibitem{lunn2012bugs}
D.~Lunn, C.~Jackson, N.~Best, D.~Spiegelhalter, and A.~Thomas.
\newblock {\em The BUGS book: A practical introduction to Bayesian analysis}.
\newblock Chapman and Hall/CRC, 2012.

\bibitem{marden2013model}
J.~R. Marden, S.~D. Ruben, and L.~Y. Pao.
\newblock A model-free approach to wind farm control using game theoretic
  methods.
\newblock {\em IEEE Transactions on Control Systems Technology},
  21(4):1207--1214, 2013.

\bibitem{floris_2019}
NREL.
\newblock {FLORIS. Version 1.0.0}, 2019.

\bibitem{robert2007bayesian}
C.~Robert.
\newblock {\em The Bayesian choice: from decision-theoretic foundations to
  computational implementation}.
\newblock Springer Science \& Business Media, 2007.

\bibitem{russo2014learning}
D.~Russo and B.~Van~Roy.
\newblock Learning to optimize via posterior sampling.
\newblock {\em Mathematics of Operations Research}, 39(4):1221--1243, 2014.

\bibitem{scharpff2016solving}
J.~Scharpff, D.~M. Roijers, F.~A. Oliehoek, M.~T. Spaan, and M.~M. de~Weerdt.
\newblock Solving transition-independent multi-agent {MDP}s with sparse
  interactions.
\newblock In {\em AAAI 2016: Proceedings of the Thirtieth AAAI Conference on
  Artificial Intelligence}, 2016.

\bibitem{stranders2012dcops}
R.~Stranders, L.~Tran-Thanh, F.~M.~D. Fave, A.~Rogers, and N.~R. Jennings.
\newblock {DCOP}s and bandits: {E}xploration and exploitation in decentralised
  coordination.
\newblock In {\em Proceedings of the 11th International Conference on
  Autonomous Agents and Multiagent Systems (AAMAS)}, pages 289--296.
  International Foundation for Autonomous Agents and Multiagent Systems, 2012.

\bibitem{taylor2011distributed}
M.~E. Taylor, M.~Jain, P.~Tandon, M.~Yokoo, and M.~Tambe.
\newblock Distributed on-line multi-agent optimization under uncertainty:
  Balancing exploration and exploitation.
\newblock {\em Advances in Complex Systems}, 14(03):471--528, 2011.

\bibitem{thompson1933likelihood}
W.~R. Thompson.
\newblock On the likelihood that one unknown probability exceeds another in
  view of the evidence of two samples.
\newblock {\em Biometrika}, 25(3/4):285--294, 1933.

\bibitem{vandijk2016}
M.~T. Van~Dijk, J.~W. Wingerden, T.~Ashuri, Y.~Li, and M.~Rotea.
\newblock Yaw-misalignment and its impact on wind turbine loads and wind farm
  power output.
\newblock {\em Journal of Physics: {C}onference Series}, 753(6), 2016.

\bibitem{vershynin2018high}
R.~Vershynin.
\newblock {\em High-dimensional probability: An introduction with applications
  in data science}, volume~47.
\newblock Cambridge University Press, 2018.

\bibitem{verstraeten2019fleetwide}
T.~Verstraeten, A.~Nowe, J.~Keller, Y.~Guo, S.~Sheng, and J.~Helsen.
\newblock Fleetwide data-enabled reliability improvement of wind turbines.
\newblock {\em Renewable and Sustainable Energy Reviews}, 109:428--437, 2019.

\bibitem{vlassis2004anytime}
N.~Vlassis, R.~Elhorst, and J.~R. Kok.
\newblock Anytime algorithms for multiagent decision making using coordination
  graphs.
\newblock In {\em IEEE International Conference on Systems, Man and
  Cybernetics}, volume~1, pages 953--957, 2004.

\bibitem{wiering2000multi}
M.~Wiering.
\newblock Multi-agent reinforcement learning for traffic light control.
\newblock In {\em Machine Learning: Proceedings of the Seventeenth
  International Conference (ICML'2000)}, pages 1151--1158, 2000.

\bibitem{yokoo1998distributed}
M.~Yokoo, E.~H. Durfee, T.~Ishida, and K.~Kuwabara.
\newblock The distributed constraint satisfaction problem: Formalization and
  algorithms.
\newblock {\em IEEE Transactions on knowledge and data engineering},
  10(5):673--685, 1998.

\end{thebibliography}

\section*{Acknowledgments}
The authors would like to acknowledge FWO (Fonds Wetenschappelijk Onderzoek) for their support through the SB grants of Timothy Verstraeten (\#1S47617N), Eugenio Bargiacchi (\#1SA2820N) and Pieter JK Libin (\#1S31916N). Diederik M Roijers was a Postdoctoral Fellow with the FWO (grant \#12J0617N). This research was supported by funding from the Flemish Government under the ``Onderzoeksprogramma Artifici\"ele Intelligentie (AI) Vlaanderen'' programme and under the VLAIO Supersized 4.0 ICON project.

\end{document}